\documentclass[twoside]{article}

\usepackage[accepted]{aistats2021}
\usepackage[round]{natbib}

\usepackage{amsmath,amssymb,amsthm}
\usepackage{graphicx}
\usepackage{xcolor}
\usepackage{xfrac}
\usepackage{enumitem}
\usepackage{hyperref}
\usepackage{footmisc}
\usepackage[ruled,vlined]{algorithm2e}

\newcommand{\B}{\mathcal{B}} % Besov space
\newcommand{\C}{\mathcal{C}} % Holder space
\newcommand{\E}{\mathop{\mathbb{E}}} % Expectation operator
\newcommand{\F}{\mathcal{F}} % Function class
\renewcommand{\L}{\mathcal{L}} % Lebesgue space
\newcommand{\N}{\mathbb{N}} % Natural numbers
\renewcommand{\P}{\mathcal{P}} % Partition
\newcommand{\R}{\mathbb{R}} % Real numbers
\newcommand{\subG}{\mathsf{subG}(\eta^2)} % Sub-Gaussian distributions with variance proxy \eta^2
\newcommand{\X}{\mathcal{X}} % Sample space
\newcommand{\Z}{\mathcal{Z}} % Randomness space
\newcommand{\IID}{\stackrel{IID}{\sim}} % Independent and identically distributed

\renewcommand{\hat}{\widehat} % Make hats wide by default
\theoremstyle{plain}
\newtheorem{theorem}{Theorem}
\newtheorem{lemma}[theorem]{Lemma}

\theoremstyle{definition}
\newtheorem{remark}[theorem]{Remark}
\newtheorem{definition}[theorem]{Definition}

\begin{document}

% \runningtitle{Continuum-Armed Bandits: A Function Space Perspective}


\twocolumn[

\aistatstitle{Continuum-Armed Bandits: A Function Space Perspective}

\aistatsauthor{ Shashank Singh }

\aistatsaddress{ \texttt{shashankssingh44@gmail.com}\\Carnegie Mellon University } ]

\begin{abstract}
    Continuum-armed bandits (a.k.a., black-box or $0^{th}$-order optimization) involves optimizing an unknown objective function given an oracle that evaluates the function at a query point, with the goal of using as few query points as possible. In the most well-studied case, the objective function is assumed to be Lipschitz continuous and minimax rates of simple and cumulative regrets are known in both noiseless and noisy settings.
    This paper studies continuum-armed bandits under more general smoothness conditions, namely Besov smoothness conditions, on the objective function. In both noiseless and noisy conditions, we derive minimax rates under simple and cumulative regrets.
    Our results show that minimax rates over objective functions in a Besov space are identical to minimax rates over objective functions in the smallest H\"older space into which the Besov space embeds.
\end{abstract}

\section{Introduction}

Many computational problems involve optimizing an unknown function of continuous inputs given only the ability to evaluate the function at a point. In settings where the function is expensive (in terms of computation, time, money, or another resource) to evaluate, it is desirable to estimate the global optimum of the function using as few evaluations as possible.
This problem, which we refer to as \emph{continuum-armed bandits} (following \citet{agrawal1995continuum}), is also known as \emph{black-box}~\citep{golovin2017google}, \emph{derivative-free}~\citep{larson2019derivative}, or \emph{$0^{th}$-order}~\citep{golovin2019gradientless} optimization, due to the unavailability of analytic information, including gradients, about the function.

Recent interest in continuum-armed bandits has been fueled by a need to tune hyperparameters of machine learning models and other complex algorithms~\citep{li2017hyperband,kandasamy2018tuning,feurer2019hyperparameter}.
However, the problem is a natural framework for improving many scientific and engineering processes as diverse as scientific experimentation~\citep{nakamura2017design}, taxonomy matching~\citep{pandey2007bandits}, and baking cookies~\citep{kochanski2017dessert}.

To make continuum-armed bandits tractable, most work assumes the ``arms'' (i.e., values being optimized over) lie in a compact set $\X \subset \R^d$ and makes smoothness assumptions on the objective function. For example, much work has focused on the case where the objective function lies in a Lipschitz or H\"older space over $\X$.
However, compared to our understanding of closely related estimation problems such as nonparametric regression, our understanding of the relationship between the assumptions made on the objective function and the difficulty of optimization is limited.

This paper considers a large family of function spaces, called Besov spaces, which include previously considered function spaces, such as H\"older and Sobolev spaces, as special cases. Classical results for nonparametric regression shown that convergence rates over Besov spaces can sometimes be improved by used special ``spatially adaptive'' methods, such as certain wavelet or spline methods, instead of commonly used ``linear'' methods, such as most kernel methods~\citep{donoho1995asymptopia,donoho1998minimax}. Our primary motivation is to investigate whether such an improvement is possible for continuum-armed bandits. To understand this, we prove new lower bounds on minimax convergence rates for continuum-armed bandits over Besov spaces under a variety of conditions. Our lower bounds imply that minimax rates over Besov spaces are identical to those obtained by embedding the Besov space in a H\"older space, and we conclude that, unlike in regression, Besov smoothness provides no advantage over H\"older smoothness, clarifying the nature of the smoothness properties of a function that determine the difficulty of its optimization. Additionally, as we discuss in Section~\ref{sec:related_work}, our results generalize, unify, and strengthen several prior results.

\paragraph{Organization} The remainder of this paper is organized as follows: Section~\ref{sec:notation} provides necessary notation, defines relevant function spaces, and formalizes the optimization problem we study. Section~\ref{sec:related_work} reviews related work, primarily work in the H\"older and reproducing kernel Hilbert space (RKHS) cases. Sections \ref{sec:noiseless_case} and \ref{sec:noisy_case} state and prove our main results, for the noiseless and noisy cases, respectively, and Section~\ref{sec:conclusion} concludes with some discussion.

\section{Notation, Problem Setup, and Technical Background}
\label{sec:notation}

\paragraph{Notation} For non-negative real sequences $\{a_n\}_{n \in \N}$, $\{b_n\}_{n \in \N}$, $a_n \lesssim b_n$ indicates $\limsup_{n \to \infty} \frac{a_n}{b_n} < \infty$, and $a_n \asymp b_n$ indicates $a_n \lesssim b_n$ and $b_n \lesssim a_n$.
For a real number $\sigma$, $\lfloor \sigma \rfloor := \max \{s \in \Z : s < \sigma\}$ denotes the greatest integer strictly less than $\sigma$.
For non-negative integers $k \in \N$, $[k] := \{1,...,k\}$ denotes the set of positive integers $\leq k$. For the remainder of this paper, we fix the domain of the optimization problem to be the $d$-dimensional unit cube, denoted $\X := [0, 1]^d$.

\subsection{H\"older and Besov Spaces}

We now define the main function spaces and norms discussed in this paper, and review basic relationships between these spaces. We begin with H\"older spaces:

\begin{definition}[H\"older Seminorm, Space, and Ball]
Fix $s \in [0, \infty)$. When $s$ is an integer, $\C^s(\X)$ denotes the class of $s$-times differentiable functions in $\L_\infty(\X)$. Otherwise, the \emph{H\"older seminorm} $\|\cdot\|_{\C^s(\X)} : \C^{\lfloor s \rfloor} \to [0, \infty]$ is defined by
\[\|f\|_{\C^s(\X)}
  = \sup_{\beta \in \N^d : \|\beta\| = \lfloor s \rfloor} \sup_{x \neq y \in \X} \frac{\left| f^\beta(x) - f^\beta(y) \right|}{\|x - y\|_2^{s - \lfloor s \rfloor}},\]
for all $f \in \C^{\lfloor s \rfloor}(\X)$, where $f^\beta$ denotes the $\beta^{th}$ mixed derivative of $f$. The \emph{H\"older space} $\C^s(\X)$ is defined by
\[\C^s(\X) := \{f \in \L_\infty : \|f\|_{\C^s(\X)} < \infty\}\]
and, for $L > 0$, the \emph{H\"older ball} $\C^s(\X, L)$ is defined by
\[\C(\X, L) := \{f \in \L_\infty : \|f\|_{\C^s(\X)} \leq L\}.\]
\end{definition}

We now turn to defining Besov spaces. While many equivalent definitions of Besov spaces exist, we use a definition in terms of orthonormal wavelet bases that has classically been used in the context of nonparametric statistics~\citep{donoho1995asymptopia}. This wavelet basis can be constructed from a ``father'' wavelet $\phi \in \L_\infty(\X)$ by defining, for each $j \in \N$ and $\lambda \in [2^j]^d$ (i.e., $\lambda$ is a $d$-tuple of integers between $1$ and $2^j$),
\[\phi_{j,\lambda}(x) := \left\{
    \begin{array}{ccl}
        2^{dj/2} \phi(2^{dj}x - \lambda) & : & \text{if } 2^{dj}x - \lambda \in \X \\
        0 & : & \text{ else}.
    \end{array}
\right.\]
for all $x \in \X$. In particular, the set $\{\phi_{j,\lambda}\}_{j \in \N, \lambda \in [2^j]^d}$ is a basis of $\L^2(\X)$. We will additionally require the basis to be orthonormal and to satisfy a certain regularity condition called ``$r$-regularity'', for which the father wavelet must be carefully constructed; we refer the reader to \citet{meyer1992wavelets} or \citet{daubechies1992wavelets} for standard constructions. For this paper, the exact wavelet construction is not crucial as it affects only constant factors in our rates, and it suffices to note that (a) such an orthonormal basis exists, (b) it can be indexed as $\left\{ \phi_{j, \lambda} \right\}_{j \in \N, \lambda \in [2^j]^d}$, and (c), for each $j \in \N$, $\lambda \neq \lambda' \in [2^j]^d$, $\phi_{j,\lambda}$ and $\phi_{j,\lambda'}$ have disjoint supports. Given such a wavelet basis, we can define Besov spaces and norms as follows:

\begin{definition}[Besov Norm, Space, and Ball]
Fix $\sigma \in (0, \infty)$ and $p,q \in [1, \infty]$.
Given a wavelet basis $\{\phi_{j, \lambda}\}_{j \in \N, \lambda \in [2^j]^d}$ of $\L_2(\X)$, the \emph{Besov norm} $\|\cdot\|_{\B_{p,q}^\sigma(\X)} : \L_2(\X) \to [0, \infty]$ is defined by
\begin{equation}
    \|f\|_{\B_{p,q}^\sigma(\X)}
    = \left( \sum_{j \in \N} 2^{jq(\sigma+d\frac{p - 2}{2p})} \left( \sum_{\lambda \in [2^j]^d} f_{j,\lambda}^p \right)^{\frac{q}{p}}
		      \right)^{\frac{1}{q}},
	\label{eq:besov_norm}
\end{equation}
for all $f \in \L_2(\X)$, where $f_{j,\lambda} := \langle f, \phi_{j,\lambda} \rangle_{\L_2(\X)}$ denotes the $(j,\lambda)^{th}$ coefficient in the wavelet expansion of $f$.
The \emph{Besov space} $\B_{p,q}^\sigma(\X)$ is defined by
\[\B_{p,q}^\sigma(\X) := \{f \in \L_2 : \|f\|_{\B_{p,q}^\sigma(\X)} < \infty\}\]
and, for $L > 0$, the \emph{Besov ball} $\B_{p,q}^\sigma(\X, L)$ is defined by
\[\B_{p,q}^\sigma(\X, L) := \{f \in \L_2 : \|f\|_{\B_{p,q}^\sigma(\X)} \leq L\}.\]
\end{definition}

Besov spaces are appealing to work with in part because they provide a unified representation of many commonly used function spaces.
Examples include
H\"older spaces $\C^s(\X) = \B_{\infty,\infty}^s(\X)$,
Hilbert-Sobolev spaces $\B_{2,2}^s(\X)$ (which also correspond to the Mat\'ern kernel RKHS of order $s - d/2$), and the space $\text{BV}$ of functions of bounded variation, which satisfies $\B_{1,1}^1 \subseteq \text{BV} \subseteq\B_{1,\infty}^1$ \footnote{As will become clear from our main results, the parameter $q$ of the Besov space does not affect minimax convergence rates, and so an inclusion of the form $\B_{p,1}^\sigma \subseteq \mathcal{S} \subseteq\B_{p,\infty}^\sigma$ is effectively an equivalence for our purposes.}.
For standard references on Besov spaces, see \citet{sawano2018theory} or Chapter 14 of \citet{leoni2017first}.

\begin{figure*}
    \centering
    \includegraphics[width=\linewidth,trim=8mm 0mm 0mm 0mm,clip]{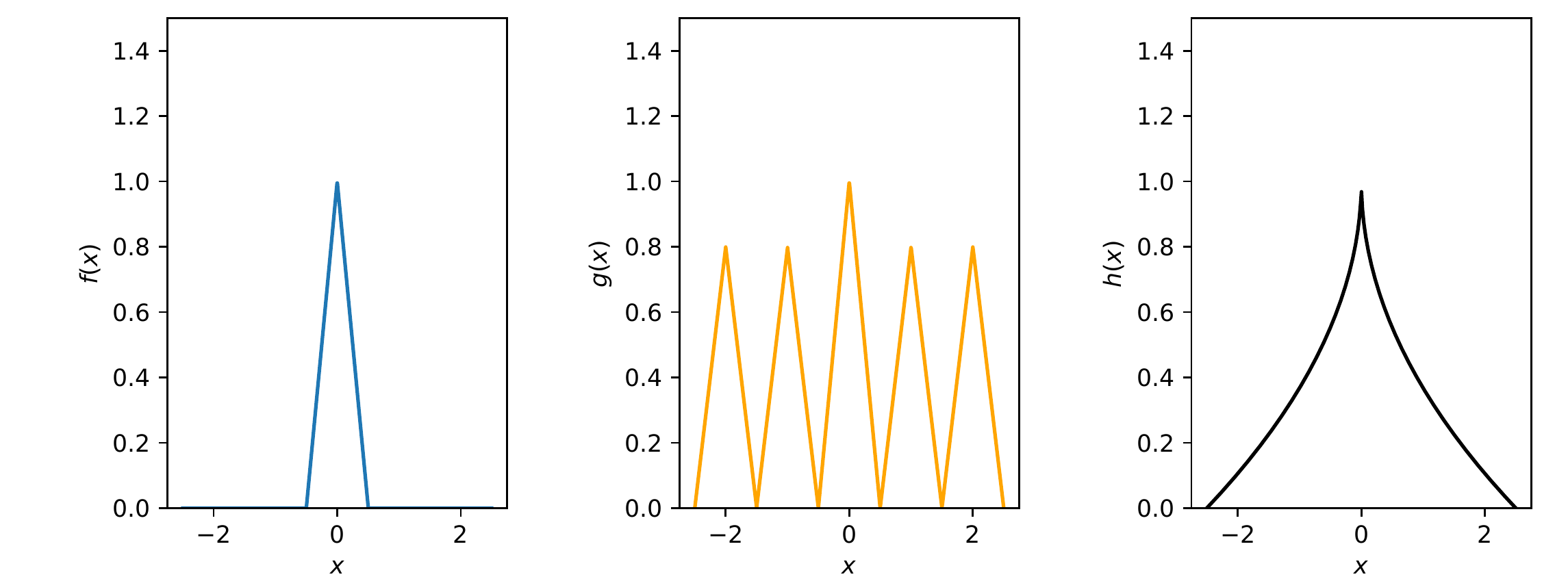}
    \caption{Three example functions, $f, g, h : \X = [-2.5, 2.5] \to [0, 1]$ with maxima at $x = 0$. $f$ and $g$ have the same Lipschitz norm ($\|f\|_{\C^1(\X)} = \|g\|_{\C^1(\X)}$), but $g$ has a larger Besov $\B_{p,q}^1$ norm ($\|f\|_{\B_{p,p}^1(\X)} < \|g\|_{\B_{p,p}^1(\X)}$) for $p < \infty$. The function $h$ has infinite Lipschitz norm ($\|h\|_{\C^1(\X)} = \infty$), but has finite $\tfrac{1}{2}$-H\"older and $1$-Sobolev norms ($\|h\|_{\C^{1/2}(\X)}, \|h\|_{\B^1_{2,2}(\X)} < \infty$).}
    \label{fig:lipschitz_versus_besov}
\end{figure*}

To illustrate the relative behaviors of H\"older and other Besov smoothness conditions, Figure~\ref{fig:lipschitz_versus_besov} shows an example of two functions with identical Lipschitz norm and different Besov norms.
Compared to H\"older smoothness, which is determined only by the least smooth part of a function, Besov smoothness is determined by a function's integrated smoothness over the entire domain. For example, for a differentiable function, the total variation norm is equivalent to the average magnitude of the gradient of the function~\citep{leoni2017first}. However, since both norms ultimately measure smoothness, Besov norms also exhibit certain relationships with H\"older norms, exemplified by the following continuous embedding of Besov spaces into H\"older spaces, which will be utilized later in this paper:
\begin{theorem}[Special Case of Theorem 1.15, \citet{haroske2009besov}]
Let $\sigma \in (d/p, \infty)$, $p, q \in [1, \infty]$. Then, $\B_{p,q}^\sigma(\X) \subseteq \C^{\sigma - d/p}(\X)$ and
\[\sup_{f \in \B_{p,q}^\sigma(\X) : f \neq 0} \frac{\|f\|_{\C^{\sigma - d/p}(\X)}}{\|f\|_{\B_{p,q}^\sigma(\X)}} < \infty.\]
Moreover, for any $s > \sigma - d/p$, $\B_{p,q}^\sigma(\X) \not\subseteq \C^s(\X)$.
\label{thm:besov_embedding}
\end{theorem}

\subsection{Problem Setup}

Fix real values $p, q \in [1, \infty]$ and $s \in (0, \infty)$, and fix a Besov class $\B_{p,q}^\sigma(\X, L)$ over $X := [0, 1]^d$. Fix an unknown function $f \in \B_{p,q}^\sigma(\X, L)$. To allow randomized optimization algorithms, also fix a random variable $Z$ taking values in a set $\Z$. Finally, fix an evaluation time horizon $T \in \N$ (which may be known or unknown).

Formally, for $t \in [T]$, a sequential $0^{th}$-order optimization procedure $\hat X = (\hat X_1, ..., \hat X_T)$ can be characterized as a sequence of queries $\hat X_{n+1} : \Z \times \X^t \times \R^d \to \X$ mapping the previous queries $(X_1,...,X_t)$ and their observed costs $Y_1 = f(X_1),...,Y_t = f(X_t)$ to a new query point.
For notational compactness, we will simply write $\hat X_{t+1} = \hat X_{t+1}(Z,X_1,...,X_t,Y_1,...,Y_t)$, with the understanding that $\hat X_{t+1}$ may depend on algorithmic randomness $Z$, previous queries $X_1,...,X_t$, and observations $Y_1,...,Y_t$.

In this paper, we will separately consider two problem settings, namely the \textit{noiseless case}, in which we query the objective $f$ directly, and the \textit{noisy case}, in which observations are corrupted by IID additive noise.

In the noiseless case, the goal is to minimize either the expected simple regret
\[R_S(\hat X, f) := \E_Z \left[ \sup_{x \in \X} f(x) - f(X_T)\right]\]
or the expected cumulative regret
\[R_C(\hat X, f) := \E_Z \left[ \sum_{n = 1}^T \sup_{x \in \X} f(x) - f(X_n)\right].\]
Specifically, we are interested in bounding (up to constant factors) the minimax risk
\begin{align*}
    M_{S,T}(\F) & := \inf_{\hat x, Z} \sup_{f \in \F} R_S(\hat x_T, f) \\
    \text{ or } \quad
    M_{C,T}(\F) & := \inf_{\hat x, Z} \sup_{f \in \F} R_C(\hat X_T, f),
\end{align*}
where the infima are taken over all optimization procedures $\hat X$ and randomness sources $Z$.

In the noisy case, we additionally fix a noise probability distribution $P$, and, rather than observing the true value of $f$ at the query point, our observations are corrupted by additive noise sampled IID from $P$; i.e., $Y_t = f(X_t) + \epsilon_t$, where $\epsilon_1,...,\epsilon_T \IID P$.
Our goal is again to minimize either the expected simple regret
\[R_S(\hat X, f, P)
  := \E_{Z, \epsilon} \left[ \sup_{x \in \X} f(x) - f(X_T)\right]\]
or the expected cumulative regret
\[R_C(\hat X, f, P)
  := \E_{Z, \epsilon} \left[ \sum_{n = 1}^T \sup_{x \in \X} f(x) - f(X_n)\right],\]
where the expectations are now additionally taken with respect to the noise $\epsilon$. Assuming further that the unknown distribution $P$ lies in some regularity class $\P$, we are interested in characterizing the minimax risk
\begin{align*}
    M_{S,T}(\F, \P) & := \inf_{\hat x, Z} \sup_{f \in \F, P \in \P} R_S(\hat X_T, f, P) \\
    \text{ or } \quad
    M_{C,T}(\F, \P) & := \inf_{\hat x, Z} \sup_{f \in \F, P \in \P} R_C(\hat X_T, f, P).
\end{align*}

In particular, we shall consider the case when $\P = \subG$ is the class of sub-Gaussian random variables with variance proxy $\eta^2$.

\section{Related Work}
\label{sec:related_work}
Bandit optimization of smooth functions is a well-studied problem with a number of theoretical and practical results, with perhaps the earliest results due to \citet{kiefer1952stochastic}, who proposed a predecessor of stochastic gradient descent based on finite-differences.
Here, we review the most relevant modern work, consisting primarily of work on the cases of objective functions lying in a H\"older space or RKHS.

First, we note that, under \emph{simple regret}, upper bounds can be obtained using ``pure exploration'' algorithms that randomly sample $\X$, apply standard (e.g., local polynomial~\citep{fan1996local} or wavelet~\citep{donoho1998minimax}) nonparametric regression methods, and then compute the optimum of the estimated regression function. In the H\"older case, it has long been known that pure exploration algorithms can achieve the minimax optimal rate, and recent work has therefore focused on showing that more sophisticated adaptive algorithms strictly outperform pure exploration under stronger assumptions (e.g., local convexity or other shape constraints near the optimum) on the objective function~\citep{minsker2012non,minsker2013estimation,grill2015black,wang2018optimization}. Our lower bounds confirm that, in both noiseless and noisy conditions, in terms of simple regret, pure exploration continues to be minimax rate-optimal under more general Besov smoothness assumptions, at least in the absence of further shape constraints.

Under \emph{cumulative regret}, a number of papers have studied the case when the objective function lies in the H\"older class $\C^s$~\citep{kleinberg2005nearly,kleinberg2008multi,bubeck2011x,bubeck2011lipschitz} (see Chapter 4 of \cite{slivkins2019introduction} for further review).
For example, in the $s$-H\"older case with $s \in (0, 1]$, if $\P$ is any class of uniformly bounded noise distributions, \citet{bubeck2011x} showed upper bounds on cumulative regret of order
\begin{equation}
    M_{C,T} \left( \C^s(\X, L), \P \right) \lesssim T^{\frac{s + d}{2s + d}} \left( \log T \right)^{\frac{1}{2s + d}}.
    \label{ineq:noisy_holder_cumulative_rate}
\end{equation}
This rate matches lower bounds such as ours in $T$, which, together with the embedding Theorem~\ref{thm:besov_embedding}, implies that the rates are minimax-optimal.
Our lower bounds for the noisy case match~\eqref{ineq:noisy_holder_cumulative_rate} when $s \in (0, 1]$, showing that this rate is optimal not only for H\"older classes but also for more general Besov classes.
We note that \citet{bubeck2011x} additionally showed that, for the upper bound~\eqref{ineq:noisy_holder_cumulative_rate}, Lipschitz smoothness (i.e., with $s = 1$) is needed only in weak local form, near the global optimum.

To the best of our knowledge, under cumulative regret, the $s$-H\"older case with $s > 1$ is open, in that a gap remains existing lower and upper bounds\footnote{For the significantly different setting where the objective function is strongly convex, \citet{akhavan2020exploiting} recently derived a dimension-independent minimax cumulative regret rate of $\asymp T^{1/s}$ that applies for $s > 1$, in the presence of noise.}. Our lower bounds extend smoothly from $s \in (0, 1]$ to $s > 1$, and we conjecture that our lower bounds are tight for all $s > 0$. However, existing upper bounds for $s > 1$ are polynomially larger in $T$. In particular, \citet{grant2020thompson} consider the H\"older case when $s$ is an integer (i.e., when the $(s - 1)^{st}$ derivative of the objective function is Lipschitz).
Under a particular prior on the objective function, they upper bounded the Bayesian cumulative risk of a Thompson Sampling~\citep{thompson1933likelihood} algorithm by order $T^{\frac{2s^2 + 7s + 1}{4s^2 + 6s + 2}}$. To the best of our knowledge, these are the only known upper bounds on cumulative regret that improve with $s > 1$, and they approach the parametric rate of $T^{1/2}$ as $s \to \infty$. However these rates are polynomially worse than the minimax rate for all finite $s$, and, moreover, they are proven for Bayesian risk, a weaker notion of risk than worst-case/minimax risk~\citep{wasserman2013all}.
\footnote{Concurrent with the present article, \cite{liu2020smooth} provide an algorithm, based on running a UCB meta-algorithm over a collection of local polynomial bandit algorithms, and bound the worst-case risk of this algorithm in the presence of noise for all H\"older exponents $s > 0$. Their upper bound exceeds our lower bound by a multiplicative factor of $\left( \log T \right)^{\frac{5s + 2d}{2s + d}}$, providing the minimax rate, up to polylogarithmic factors in $T$, for all $s > 0$.}

Beyond the H\"older setting, work has largely focused on the setting in which $\F$ is an RKHS. Due to the equivalence between RKHSs of Mat\'ern kernels and Hilbert-Sobolev spaces (see, e.g., Example 2.6 of \citet{kanagawa2018gaussian}) this intersects with our results in the special case $p = q = 2$. In particular, for $\sigma > d/2$, in the absence of noise, \citet{bull2011convergence} obtained a simple regret rate
\begin{equation}
    M_{S,T} \left( \B_{2,2}^\sigma(\X, L) \right) \asymp T^{1/2 - \sigma/d},
    \label{ineq:noiseless_holder_simple_rate}
\end{equation}
while, in the presence of noise, \citet{scarlett2017lower} proved lower bounds of orders
\[M_{S,T} \left( \B_{2,2}^\sigma(\X, L), \subG \right)
  \gtrsim \left( \frac{\eta^2}{T} \right)^{\frac{\sigma - d/2}{2\sigma}}\]
and
\[M_{C,T} \left( \B_{2,2}^\sigma(\X, L), \subG \right)
  \gtrsim \eta^{\frac{\sigma - d/2}{\sigma}} T^{\frac{\sigma + d/2}{2\sigma}},\]
under simple and cumulative regrets, respectively. Our results generalize these lower bounds to other Besov classes with $p, q \neq 2$, and, in the noisy case, tighten these bounds (by polylogarithmic factors in $T$ and $\eta$) to the correct minimax rate. Moreover, although this was less apparent in the context of RKHSs (and appears to have been missed by \citet{scarlett2017lower}), expressing this problem in the language of Besov spaces makes it clear that these rates can be directly compared to those for the H\"older case. Prior to this, the best known upper bound appears to have been a rate of $T^{\frac{\sigma + d^2 + d/2}{2\sigma + d^2}}$ due to \citet{srinivas2010gaussian} based on a Gaussian Process UCB algorithm; this rate is sub-optimal by polynomial factors in $T$.

On a more historical note, one of the primary motivators for this work is a series of classical results due to \citet{nemirovskii1985nonparametric}, \citet{donoho1998minimax}, and others, who investigated nonparametric \emph{estimation} problems, such as regression and density estimation, over Besov spaces.
They showed that minimax rates for nonparametric estimation over $\B_{p,q}^\sigma(\X)$ are significantly faster than over $\C^{\sigma - d/p}(\X)$.
However, in contrast to minimax optimality over $\C^{\sigma - d/p}(\X)$, for which simpler methods (e.g., kernel, basis, or spline regression) are optimal, leveraging accelerated convergence rates over $\B_{p,q}^\sigma(\X)$ requires relatively complex methods (such as thresholding wavelet methods or spatially adaptive smoothing splines) that are ``spatially adaptive'', in the sense of devoting selectively greater representational power to less smooth parts of the estimand.
The equivalence we show for \emph{optimization} over Besov and H\"older spaces shows that this phenomenon fails to carry over from estimation; in particular, spatial adaptivity confers no asymptotic benefit for bandit optimization over Besov spaces. This may be related to the finding of \citet{bubeck2011x} that Lipschitz smoothness may be needed only in a weak sense, locally near the global optimum; i.e., allocation of representational power to the function near a point $x$ should correspond to the plausibility of that point being a global optimum, rather than to relative smoothness of $f$ near that point.
% It is not obvious how to generalize this finding beyond the H\"older case, and we leave it to future work to investigate this hypothesis.

Finally, we note that many papers have also considered stronger shape assumptions such as convexity of the objective function, either globally or near an optimum~\cite{auer2007improved,cope2009regret,bubeck2018sparsity,wang2018convex,golovin2019gradientless}. In contrast to the rates discussed above and elsewhere in this paper, which scale exponentially with the number $d$ of optimization variables, rates in these cases scale much more favorably (polynomially) with $d$. However, convexity and related shape assumptions may be too strong for many black-box optimization settings.

To summarize, \textbf{our main contributions} are:
\begin{enumerate}
    \item We identify upper bounds on the regret of existing algorithms under Besov smoothness conditions by embedding the Besov space in a H\"older space of lower smoothness. For many Besov spaces, these bounds are the first available, while, for others, such as Hilbert-Sobolev spaces, they improve on state-of-the-art results.
    \item We derive novel information-theoretic lower bounds on minimax regret under Besov smoothness conditions, generalizing and tightening a number of existing lower bounds. These lower bounds match the above upper bounds, thereby identifying the minimax rate for these problems.
\end{enumerate}
These results are provided in both noiseless and noisy conditions, under both simple and cumulative losses, providing a comprehensive picture for this class of problems. Moreover, as discussed in Section~\ref{sec:conclusion}, a number of big-picture conclusions follow from our results.

\section{Noiseless Case}
\label{sec:noiseless_case}

Here, we state and prove our main results, identifying minimax rates of simple and cumulative regret over Besov spaces, for the noiseless case.
\begin{theorem}[Minimax Rates, Noiseless Case]
For $\sigma > d/p$,
\begin{align*}
    M_{S,T} \left( \B_{p,q}^\sigma(\X, L) \right)
    & \asymp M_{S,T} \left( \C^{\sigma - d/p}(\X, L) \right) \\
    & \asymp T^{1/p - \sigma/d},
\end{align*}
and, for $\sigma \in (d/p, d/p + 1]$,
\begin{align*}
    M_{C,T} \left( \B_{p,q}^\sigma(\X, L) \right)
    & \asymp M_{C,T} \left( \C^{\sigma - d/p}(\X, L) \right) \\
    & \asymp T^{1 + 1/p - \sigma/d}.
\end{align*}
  \label{thm:noiseless_rates}
\end{theorem}
Before proving this result, we make a few remarks:

\begin{remark}
% Theorem~\ref{thm:noiseless_rates} considers only $\sigma \in (d/p, d/p + 1]$.
When $\sigma \leq d/p$, the picture is quite simple, as $\B_{p,q}^\sigma \left( \X, L \right)$ contains functions with unbounded singularities, and it is easy to show that $M_{S,T} \left( \B_{p,q}^\sigma \left( \X, L \right) \right) = M_{C,T} \left( \B_{p,q}^\sigma \left( \X, L \right) \right) = \infty$. When $\sigma > d/p + 1$, under cumulative regret, our lower bounds continue to hold, and we conjecture that they are tight, while upper bounds are either non-existent or loose. We leave it to future work to tighten these upper bounds.
\label{remark:singularity}
\end{remark}

\begin{remark}
Since the instantaneous regret at time $t$ of any algorithm for minimizing cumulative regret will always be at least the instantaneous regret of the the minimax-optimal procedure for minimizing simple regret at time $t$, one can check that, for any $T$, $M_{C,T} \geq \sum_{t = 1}^T M_{S,t}$. Also, given an algorithm $\hat X$ with expected cumulative regret $R$, one can construct an algorithm with expected simple regret $\frac{R}{T - 1}$ by sampling $X_T$ 
uniformly at random from the first $T - 1$ queries of $\hat X$. Thus, in what follows, we prove lower bounds on simple regret and upper bounds on cumulative regret, while lower bounds on cumulative regret and upper bounds on simple regret follow via these inequalities.
\label{remark:simple_cumulative}
\end{remark}

We now turn to proving Theorem~\ref{thm:noiseless_rates}. We begin by constructing a ``worst-case'' subset of Besov functions:

\begin{lemma}[Construction of $\Theta_j$]
    For any positive integer $j \in \N$, there exists a subset $\Theta_j \subseteq \B_{p,q}^\sigma(\X, L)$ of Besov functions with the following properties:
    \begin{enumerate}[label=(\alph*)]
        \item $|\Theta_j| = 2^{d \cdot j}$.
        \item For $f \neq g \in \Theta_j$, $f$ and $g$ have disjoint supports.
        \item For all $f \in \Theta_j$, $\sup_{x \in \X} f(x) = L \phi^* 2^{j(d/p - \sigma)}$, where $\phi^* = \sup_{x \in \X} \phi(x)$ is a constant depending only on the choice of father wavelet.
        \item For all $f \in \Theta_j$,
        \[\max_{f \in \Theta_j} \|f\|_\infty = L \|\phi\|_\infty 2^{j(d/p - \sigma)}.\]
    \end{enumerate}
    \label{lemma:Theta_construction}
\end{lemma}

\begin{proof}
Let $\Theta_j := \{L 2^{-j(\sigma+d(1/2-1/p))} \phi_{j, \lambda} : \lambda \in [2^j]^d\}$ be a rescaling of the wavelet basis functions at resolution $j$ by the constant $L 2^{-j(\sigma+d(1/2-1/p))}$. Properties (a) and (b) are immediate from the construction of the wavelet basis. For property (c), note that, for $f \in \Theta_j$,
\begin{align*}
  \max_{x \in \X} f(x)
  & = L \phi_{j,\lambda}^* 2^{-j(\sigma+d(1/2-1/p))} \\
  & = L (\phi^* 2^{dj/2}) 2^{-j(\sigma+d(1/2-1/p))} \\
  & = L \phi^* 2^{j(d/p - \sigma)}.
\end{align*}
Checking property (d) is nearly identical, replacing $\phi^*$ with $\|\phi\|_\infty$.
Finally, to verify $\Theta_j \subseteq \B_{p,q}^\sigma(\X, L)$, Eq.~\eqref{eq:besov_norm} gives, for $\lambda \in [2^j]^d$:
\begin{align*}
    & \|L 2^{-j(\sigma+d(1/2-1/p))} \phi_{j, \lambda}\|_{\B_{p,q}^\sigma(\X)} \\
    & = 2^{j(\sigma+d(1/2-1/p))} \langle L 2^{-j(\sigma+d(1/2-1/p))} \phi_{j, \lambda}, \phi_{j, \lambda} \rangle = L.
\end{align*}
\end{proof}

Lemma~\ref{lemma:Theta_construction} constructed a large family $\Theta_j$ of objective functions with disjoint supports. We now use this lemma to prove the main Theorem~\ref{thm:noiseless_rates}. The strategy is to show, based on the size of $\Theta_j$, that, for any optimization strategy $\hat x$, there is some objective function $f \in \Theta_j$ such that none of the first $T$ points sampled by $\hat x$ lies in the support of $f$. Thus, when the true objective function is $f$, the simple regret incurred by $\hat x$ is $\sup_{x \in \X} f(x)$.

\begin{proof}[Proof of Theorem~\ref{thm:noiseless_rates}]
For $s \in (0, 1]$, existing results (e.g., \citet{munos2011optimistic,malherbe2017global}) imply the upper bound
$M_{C,T} \left( \C^s(\X, L) \right) \lesssim T^{1 - s/d}$
for the H\"older case. Our upper bound on $M_{C,T} \left( \B_{p,q}^\sigma(\X, L) \right)$ follows from the continuous embedding $\B_{p,q}^\sigma(\X) \hookrightarrow \C^{\sigma - d/p}(\X)$ (Theorem~\ref{thm:besov_embedding}). We now turn to lower bounds on expected simple regret.

Fix any optimization procedure $\hat x$ and any random variable $Z$, and let $\Theta = \Theta_j$ denote the set of functions constructed in Lemma~\ref{lemma:Theta_construction} with $j = \left\lceil \frac{1}{d} \log_2(2T) \right\rceil$, so that $2T \leq |\Theta| \leq 2^{d + 1}T$ and
\[\min_{f \in \Theta_j} \max_{x \in \X} f(x) \geq L \phi^* 2^{(d + 1)(1/p - \sigma/d)} T^{1/p - \sigma/d}.\]
For each $f \in \Theta$, let $E_f$ denote the event that $f(X_1) = \cdots = f(X_T) = 0$. Since the elements of $\Theta$ have disjoint support, $\sum_{f \in \Theta} 1_{E_f} \geq |\Theta| - T \geq T$, and so, since $|\Theta| \leq 2^{d + 1}T$,
\[2^{d + 1}T \cdot \max_{f \in \Theta} \Pr_Z[E_f]
  \geq \sum_{f \in \Theta} \Pr_Z[E_f]
  = \E_Z \left[ \sum_{f \in \Theta} 1_{E_f} \right] \geq T.\]
In particular, there exists a function\footnote{Note that $f_{\hat x, \mu_Z}$ depends only on the distribution $\mu_Z$ of $Z$, not on any particular instance of $Z$.} $f_{\hat x, \mu_Z} \in \Theta$ such that $\Pr_Z[E_{f_{\hat x, \mu_Z}}] \geq 2^{-(d + 1)}$. It follows that
\begin{align*}
    & M_T \left( \B_{p,q}^\sigma(\X, L) \right) \\
    & = \inf_{\hat x, Z} \sup_{f \in \B_{p,q}^\sigma(\X, L)} \E_Z \left[ \max_{x \in \X} f(x) - f(x_T) \right] \\
    & \geq \inf_{\hat x, Z} \E_Z \left[ \max_{x \in \X} f_{\hat x, \mu_Z}(x) - f_{\hat x, \mu_Z}(x_T) \right] \\
    & \geq \inf_{\hat x, Z} \Pr_Z[E_{f_{\hat x, \mu_Z}}] \cdot \max_{x \in \X} f_{\hat x, \mu_Z}(x) \\
    & \geq 2^{-(d + 1)} \cdot L \phi^* 2^{(d + 1)(1/p - \sigma/d)} T^{1/p - \sigma/d}.
\end{align*}
\end{proof}

\section{Noisy Case}
\label{sec:noisy_case}

In this section, we state and prove lower bounds on the minimax simple and cumulative regrets in the case where the observed rewards are noisy. Our approach will utilize the following version of Fano's lemma:
\begin{lemma} (Fano's Lemma; Simplified Form of Theorem 2.5 of \citet{tsybakov2008introduction})
    Fix a family $\P$ of distributions over a sample space $\X$ and fix a pseudo-metric $\rho : \P \times \P \to [0, \infty]$ over $\P$. Suppose there exist $P_0 \in \P$ and a set $\Theta \subseteq \P$ such that
    \[\frac{1}{|\Theta|} \sum_{P \in \Theta} D_{KL}(P_0,P)
      \leq \frac{\log |\Theta|}{16},\]
    where $D_{KL} : \P \times \P \to [0, \infty]$ denotes Kullback-Leibler divergence.
    Then,
    \[\inf_{\hat P} \sup_{P \in \P} \Pr \left[ \rho(P,\hat P)
      \geq \frac{1}{2} \inf_{Q, R \in \Theta} \rho(Q,R) \right] \geq 1/8,\]
    where the first $\inf$ is taken over all estimators $\hat P$.
    \label{thm:tsybakov_fano}
\end{lemma}

\begin{theorem}[Minimax Rates, Noisy Case]
    For $\sigma > d/p$,
    \begin{align*}
        & M_{S,T} \left( \B_{p,q}^\sigma(\X, L), \subG \right) \\
        & \asymp M_{S,T} \left( \C^{\sigma - d/p}(\X, L), \subG \right) \\
        & \asymp \max \left\{ \left( \frac{\eta^2 \log \frac{T}{\eta^2}}{T} \right)^{\frac{\sigma - d/p}{2(\sigma - d/p) + d}}\hspace{-2mm},
          M_{S,T} \left( \B_{p,q}^\sigma(\X, L) \right) \right\},
    \end{align*}
    and, for $\sigma \in (d/p, d/p + 1]$,
    \begin{align*}
        & M_{C,T} \left( \B_{p,q}^\sigma(\X, L), \subG \right) \\
        & \asymp M_{C,T} \left( \C^{\sigma - d/p}(\X, L), \subG \right) \\
        & \asymp \max \Bigg\{ T^{\frac{d + \sigma - d/p}{2(\sigma - d/p) + d}} \left( \eta^2 \log \frac{T}{\eta^2} \right)^{\frac{\sigma - d/p}{2(\sigma - d/p) + d}}, \\
        & \hspace{1.5cm} M_{C,T} \left( \B_{p,q}^\sigma(\X, L) \right) \Bigg\}.
    \end{align*}
      \label{thm:noisy_rates}
\end{theorem}

As in the noiseless case, consistent optimization is possible only for $\sigma > d/p$, and, while our lower bounds hold for $\sigma > d/p + 1$, existing upper bounds under cumulative regret appear loose (see Remark~\ref{remark:singularity} and Related Work in Section~\ref{sec:related_work}). Also, as noted in Remark~\ref{remark:simple_cumulative}, lower bounds on minimax cumulative regret and upper bounds on minimax simple regret follow from lower bounds on minimax simple regret and upper bounds on minimax cumulative regret, respectively.

\begin{proof}
For $s \in (0, 1]$, Corollary 1 of~\citet{auer2007improved} implies the upper bound
\[M_{C,T} \left( \C^s(\X, L), \subG \right)
  \lesssim T^{\frac{s + d}{2s + d}} \left( \eta^2 \log \frac{T}{\eta^2} \right)^{-\frac{s}{2s + d}}\]
for the H\"older case. Our upper bound on $M_{C,T} \left( \B_{p,q}^\sigma(\X, L), \subG \right)$ follows from the continuous embedding $\B_{p,q}^\sigma(\X) \hookrightarrow \C^{\sigma - d/p}(\X)$ (Theorem~\ref{thm:besov_embedding}).

% As in the noiseless case, our upper bounds follow from previous upper bounds for the H\"older case (e.g., Corollary 1 of~\citet{auer2007improved}) together with the continuous embedding $\B_{p,q}^\sigma(\X) \hookrightarrow \C^{\sigma - d/p}(\X)$ from Theorem~\ref{thm:besov_embedding}. In the remainder of the proof, we lower bound the minimax simple regret.

For proving a minimax lower bound over sub-Gaussian noise distributions with variance parameter at most $\eta^2$, we may assume homoskedastic Gaussian noise with variance $\eta^2$. Let $\Theta_j$ be as given in Lemma~\ref{lemma:Theta_construction} with $j$ specified in Eq.~\eqref{eq:choice_of_j}. Let $P_0$ denote the distribution of the observations $(Y_1,...,Y_T)$ when the true function is constant $0$ on $\X$, and, for each $f \in \Theta_j$, let $P_f$ denote the distribution of $(Y_1,...,Y_T)$ when the true function is $f$. For $f \in \Theta_j \cup \{0\}$ and each $t \in [T]$, let $P_{0,t}$ denote the conditional distribution of $(Y_i,...,Y_T)|(Y_1,...,Y_{t-1})$.
By a standard formula for KL divergence between two Gaussians and the construction of $\Theta_j$ (specifically part (d) of Lemma~\ref{lemma:Theta_construction}), the maximum information gain $D_j^*$ of a single query is:
\begin{align}
    D_j^*
    \notag
    & := \sup_{x \in \X, t \in [T], f \in \Theta_j} D_{KL} \left( P_{0,t} \| P_{f,t} \right) \\
    & = \max_{f \in \Theta_j} \frac{\|f\|_\infty^2}{2\eta^2}
      = \frac{L^2 \|\phi\|_\infty^2 2^{2j(d/p - \sigma)}}{2\eta^2}
  \label{eq:D_star}
\end{align}
Then, by the chain rule for KL divergence, for any $f \in \Theta_j \cup \{0\}$,
\begin{align*}
  & D_{\text{KL}}(P_0\|P_f) \\
  & = \sum_{t = 1}^T \E_{P_0} \left[ D(P_{0,t}\|P_{f,t}) \middle| Y_1,...,Y_{t-1} \right] \\
  & = \sum_{t = 1}^T \E_{P_0} \left[ D(P_{0,t}\|P_{f,t}) \middle| Y_1,...,Y_{t-1}, f(X_t) > 0 \right] \\
  & \hspace{30pt} \cdot \Pr_{P_0}[f(X_t) > 0|Y_1,...,Y_{t-1}] \\
  & \leq D_j^* \sum_{t = 1}^T \Pr_{P_0}[f(X_i) > 0|Y_1,...,Y_{t-1}].
\end{align*}
Since the elements of $\Theta_j$ have disjoint support,
\begin{align*}
    & \sum_{f \in \Theta_j} \Pr_{P_0}[f(X_T) > 0|Y_1,...,Y_T] \\
    & = \E_{P_0}\left[ \sum_{f \in \Theta_j} 1_{\{f(X_T) > 0\}}|Y_1,...,Y_T \right]
    \leq 1,
\end{align*}
and so, since $|\Theta_j| = 2^{dj}$ (by part (a) of Lemma~\ref{lemma:Theta_construction}),
\begin{align*}
    & \frac{1}{|\Theta_j|} \sum_{f \in \Theta_j} D_{\text{KL}}(P_0\|P_f) \\
    & \leq \frac{1}{|\Theta_j|} \sum_{f \in \Theta_j} D_j^* \sum_{i = 1}^T \Pr_0[f(X_i) > 0|Y_1,...,Y_{i-1}] \\
    & \leq \frac{D_j^* T}{|\Theta_j|}
      = \frac{L^2 \|\phi\|_\infty^2 2^{2j(d/p - \sigma)} T}{2\eta^2 2^{dj}}.
\end{align*}
Defining $r := d + 2(\sigma - d/p) > 0$ and $z := \frac{L^2 \|\phi\|_\infty^2}{2} \frac{T}{\eta^2}$,
\[\frac{1}{|\Theta_j|} \sum_{f \in \Theta_j} D_{\text{KL}}(P_0\|P_f)    
    \leq z 2^{-rj}.\]
Let
\begin{equation}
    j = \frac{1}{r} \log_2 \frac{8 r z}{d \log_2 8 r z}.
    \label{eq:choice_of_j}
\end{equation}
Since $z \to \infty$ as $\frac{T}{\eta^2} \to \infty$, for sufficiently large $\frac{T}{\eta^2}$, we have $\log_2(8rz) \geq 2 d \log_2 \log_2(8rz)$. Therefore,
\begin{align*}
    & \frac{1}{|\Theta_j|} \sum_{f \in \Theta_j} D_{\text{KL}}(P_0\|P_f)
      \leq z 2^{-rj} \\
    & = \frac{d}{8r} \log_2(8rz) \\
    & \leq \frac{d}{16r} \left( \log_2(8rz) - d \log_2 \log_2(8rz) \right) \\
    & = \frac{d}{16r} \log_2 \frac{8r z}{d \log_2(8rz)}
      = \frac{dj}{16} = \frac{\log_2 |\Theta_j|}{16}.
\end{align*}
Therefore, we can apply Fano's Lemma (Lemma~\ref{thm:tsybakov_fano}) to the class $\{P_f : f \in \Theta_j\}$ equipped with the discrete metric $\rho(P, Q) = 1_{\{P \neq Q\}}$, giving
\[\inf_{\hat x, Z} \sup_{f \in \Theta_j} \Pr \left[ \rho(P_f,P_{\hat x, Z})
    \geq \frac{1}{2} \inf_{f, g \in \Theta_j} \rho(P_f,P_g) \right] \geq \frac{1}{8},\]
where $P_{\hat x, Z}$ is any distribution ``estimate'' (i.e., any distribution computed as a function of the data $\{(X_t,Y_t)\}_{t \in [T]}$ and $Z$). In particular, since the elements of $\Theta_j$ have disjoint support, we can construct a distribution estimate $P_{\hat x, Z} := P_f$, where $f$ is the unique element of $\Theta_j$ with $f(X_T) > 0$ (if $f(X_T) = 0$ for all $f \in \Theta_j$, then $f$ can be selected at random). For this particular estimate, we have
\begin{align*}
    & \inf_{\hat x, Z} \sup_{f \in \Theta_j} \Pr \left[ P_f \neq P_{\hat x, Z} \right] \\
    & = \inf_{\hat x, Z} \sup_{f \in \Theta_j} \Pr \left[ \rho(P_f,P_{\hat x, Z}) \geq \frac{1}{2} \inf_{f, g \in \Theta_j} \rho(P_f,P_g) \right] \geq \frac{1}{8}.
\end{align*}
Therefore, since $\Theta_j \subseteq \B_{p,q}^\sigma(\X, L)$,
\begin{align*}
    & M_{S,T} \left( \B_{p,q}^\sigma(\X, L) \right)
    \geq M_{S,T} \left( \Theta_j \right) \\
    & = \inf_{\hat x, Z} \sup_{f \in \Theta_j} \E_f \left[ \sup_{x \in \X} f(x) - f(X_n) \right] \\
    & \geq \inf_{\hat x, Z} \sup_{f \in \Theta_j} \E_f \left[ L \phi^* 2^{j(d/p - \sigma)} 1_{\{f(X_n) = 0\}} \right] \\
    & = L \phi^* 2^{j(d/p - \sigma)} \inf_{\hat x, Z} \sup_{f \in \Theta_j} \Pr \left[ P_f \neq P_{\hat x, Z} \right] \\
    & \geq \frac{L \phi^* 2^{j(d/p - \sigma)}}{8}
\end{align*}
Plugging in $j$ from Equation~\ref{eq:choice_of_j} gives the desired result:
\begin{align*}
    & M_{S,T} \left( \B_{p,q}^\sigma(\X, L) \right) \\
    & \qquad \geq \frac{L \phi^*}{8} \left( \frac{4 r L^2 \|\phi\|_\infty^2 \frac{T}{\eta^2}}{d \log_2 \left( 4 r L^2 \|\phi\|_\infty^2 \frac{T}{\eta^2} \right)} \right)^{\frac{d/p - \sigma}{2(\sigma - d/p) + d}},
\end{align*}
where $r = d + 2\sigma - 2d/p > 0$ is a constant.
\end{proof}

To illustrate results of Theorem~\ref{thm:noisy_rates}, Figure~\ref{fig:rate_phase_diagram} shows a phase diagram of minimax simple regret rates, as a function of parameters $\sigma$ and $p$ of the Besov space in which the objective function is assumed to lie. The figure highlights the main conclusion of our paper: minimax rates of bandit optimization over a Besov space correspond precisely to rates over the smallest H\"older space within which that Besov space embeds.

Comparing with the noiseless simple regret rate $T^{1/p - \sigma/d}$ from Theorem~\ref{thm:noiseless_rates}, up to logarithmic factors, the noiseless rate dominates for $\eta \lesssim T^{1/p - \sigma/d}$, while the noisy rate dominates for $\eta \gtrsim T^{1/p - \sigma/d}$. This threshold matches the convergence rate in the noiseless case; a natural interpretation is that the optimal error decreases at the fast noiseless rate $T^{1/p - \sigma/d}$ until it is of the same order as the noise $\eta$, at which point convergence slows to the noisy rate $T^{\frac{d + \sigma - d/p}{2(\sigma - d/p) + d}} \eta^{\frac{2\sigma - 2d/p}{2(\sigma - d/p) + d}}$.

\begin{figure}
    \centering
    \includegraphics[width=\linewidth]{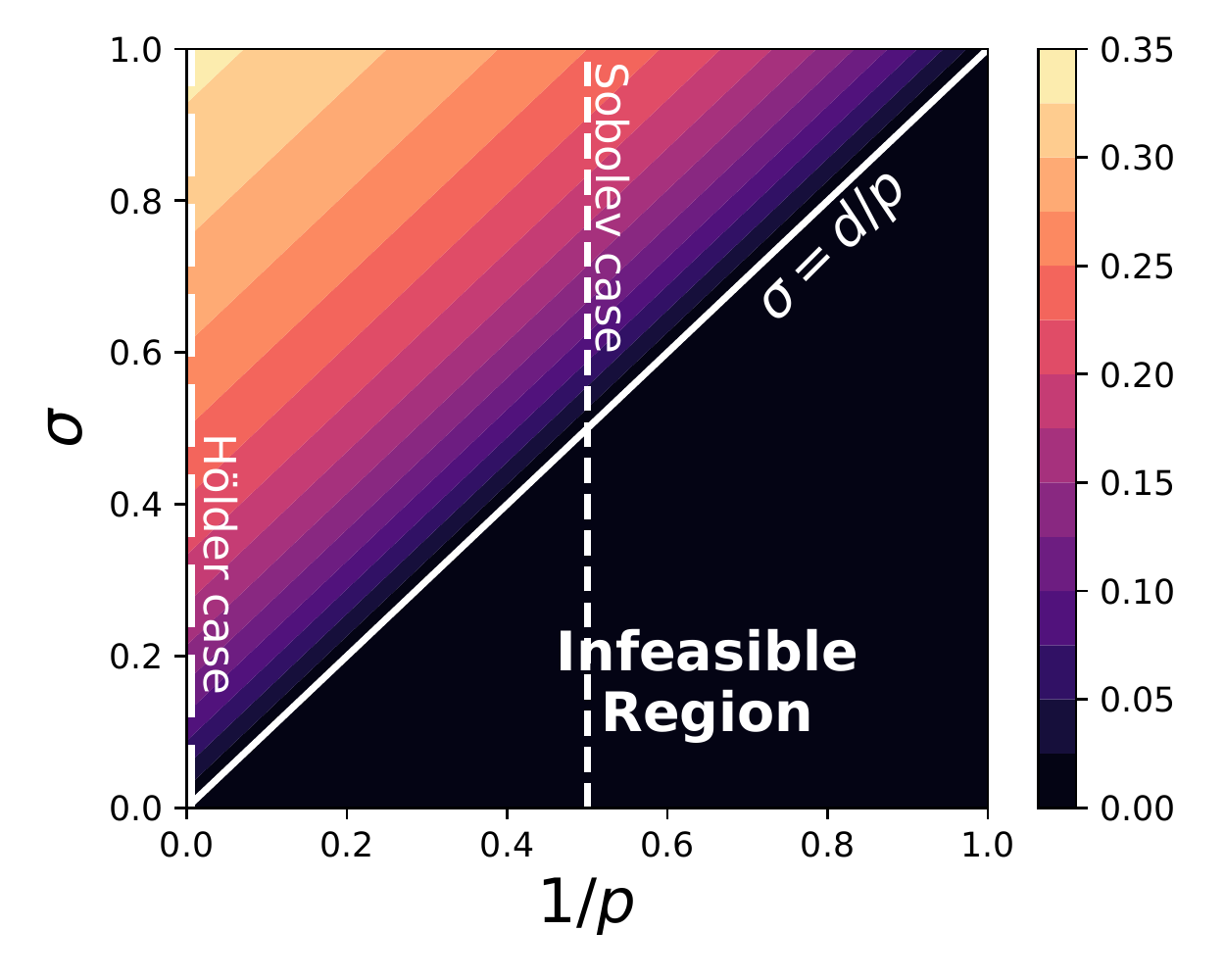}
    \caption{Phase diagram showing exponent of minimax simple regret rate in the noisy case (i.e., $\alpha$ such that $M_{S,T} \left( \B_{p,q}^\sigma(\X, L) \right) \asymp T^{-\alpha}$), plotted as a function of the Besov space parameters $\sigma$ and $1/p$, in dimension $d = 1$. Dashed lines indicate H\"older ($p = \infty$) and (Hilbert-)Sobolev ($p = 2$) special cases studied in prior work. The contours, corresponding to level sets of $\sigma - d/p$, demonstrate the relationship between Besov spaces and the embedding H\"older space.}
    \label{fig:rate_phase_diagram}
\end{figure}

\section{Conclusion}
\label{sec:conclusion}

In this paper, we derived novel lower bounds on minimax simple and cumulative regrets over general Besov spaces, in both the presence and absence of noise.
Via continuous embedding of Besov spaces in H\"older spaces, these rates match existing upper bounds, implying optimality of existing proposed algorithms and identifying optimal rates over all Besov spaces.

Our results suggest that, up to constant factors, $0^{th}$-order methods cannot benefit from the additional $d/p$ orders of smoothness in $\B_{p,q}^{s + d/p}(\X)$ over $\C^s(\X)$.
From a theoretical perspective, this suggests that H\"older spaces are ``natual'' spaces over which to study continuum-armed bandit optimization, since, optimal algorithms over H\"older spaces are automatically optimal over other Besov spaces, while the converse may not hold.
From a practical perspective, our results suggest there may be little (at most constant-factor) benefit to employing spatially adaptive methods (i.e., those that more finely represent areas of the search space where the objective is less smooth), contrasting from results in regression and other nonparametric estimation problems, where spatial adaptivity is needed to obtain minimax rates~\citep{donoho1998minimax}.

Finally, to the best of our knowledge, the results of this paper are the first to consider modeling continuum-armed bandits using wavelets. Future work should explore efficient bandit algorithms based on fast wavelet methods, which can provide significant computational advantages over kernel methods~\citep{mallat2008wavelet}.

% While we have identified minimax rates for a number of settings, rates for several closely related settings remain open.
% For example, under cumulative regret, existing upper bounds only provide the minimax optimal rate for a limited range of smoothness classes, namely $s$-H\"older smoothness with $s \leq 1$.
% For $s > 1$, existing upper bounds are polynomially (in $T$) larger than our lower bounds~\citep{qian2016kernel,grant2020thompson}.
% Due to the technical complexity of analyzing cumulative regret of a nonparametric model of higher-order smoothness, we believe that existing upper bounds are loose and conjecture that they can be tightened to match our lower bounds.

\subsubsection*{Acknowledgements}
The author would like to thank Ananya Uppal and Arkady Epshteyn for helpful insights and discussions, and Sivaraman Balakrishnan for pointing out an error in an earlier version of the paper.

\bibliographystyle{plainnat}
\bibliography{main}

\end{document}